\documentclass[a4paper,twoside]{article}
\usepackage[utf8]{inputenc}
\usepackage{tikz}
\usepackage{pgfplots}
\usepackage{tkz-euclide}
\usepackage{pgfplots}
\usetkzobj{all}
\usepackage{graphicx}
\usepackage{epsfig}
\usepackage{subcaption}
\usepackage{calc}
\usepackage{amsmath}
\usepackage{amssymb}
\usepackage{amstext}
\usepackage{amsthm}
\usepackage{multicol}
\usepackage{pslatex}
\usepackage{apalike}
\usepackage{mathtools}
\usepackage{thmtools, thm-restate}
\usepackage{bbm}
\usepackage{hyperref}
\usepackage[utf8]{inputenc}
\usepackage{float}
\usepackage{bm}
\usepackage[thinlines]{easytable}
\usepackage{cite}
\usepackage{SCITEPRESS}  
 
\newtheorem{theorem}{Theorem}[section]

\newtheorem{lemma}[theorem]{Lemma}
\theoremstyle{definition}
\newtheorem{definition}{Definition}[section]

\DeclareMathOperator{\bias}{Bias}
\DeclareMathOperator{\E}{\mathbb{E}}

\pgfplotsset{
colormap={cool}{rgb255(0cm)=(205, 230, 249); rgb255(1cm)=(0,128,255); rgb255(2cm)=(255,0,255)}
}

\newcommand{\lub}[1]{{\underset{#1}{\mathrm{sup}}} \hspace{.3em}}
\newcommand*\dif{\mathop{}\!\mathrm{d}}

\begin{document}

\title{The Bias-Expressivity Trade-off}

\author{\authorname{Julius Lauw\sup{\dagger}\orcidAuthor{0000-0003-4201-0664}, Dominique Macias\sup{\dagger}\orcidAuthor{0000-0002-6506-4094}, Akshay Trikha\sup{\dagger}\orcidAuthor{0000-0001-8207-6399}, Julia Vendemiatti\sup{\dagger}\orcidAuthor{0000-0002-6547-9601},\\ George D.\  Monta\~nez\sup{\dagger}\orcidAuthor{0000-0002-1333-4611}}
\affiliation{AMISTAD Lab, Department of Computer Science, Harvey Mudd College, Claremont, CA 91711, USA}
\email{\{julauw, dmacias, atrikha, jvendemiatti, gmontanez\}@hmc.edu}
\sup{\dagger}denotes equal authorship.
}

\keywords{Machine Learning, Algorithmic Search, Inductive Bias, Entropic Expressivity}

\abstract{Learning algorithms need bias to generalize and perform better than random guessing. We examine the  flexibility (expressivity) of biased algorithms. An expressive algorithm can adapt to changing training data, altering its outcome based on changes in its input. We measure expressivity by using an information-theoretic notion of entropy on algorithm outcome distributions, demonstrating a trade-off between bias and expressivity. To the degree an algorithm is biased is the degree to which it can outperform uniform random sampling, but is also the degree to which is becomes inflexible. We derive bounds relating bias to expressivity, proving the necessary trade-offs inherent in trying to create strongly performing yet flexible algorithms.
}

\onecolumn \maketitle \normalsize \setcounter{footnote}{0} \vfill

\section{\uppercase{Introduction}}
\label{sec:introduction}
\noindent
Biased algorithms, namely those which are more heavily predisposed to certain outcomes than others, have difficulty changing their behavior in response to new information or new training data. Yet  bias is needed for learning~\cite{futilityOfBiasFreeLearning}. Given a set of information resources (or a distribution over them), an algorithm that can output many different responses is said to be more \textit{expressive} than one that cannot. We explore the inverse relationship between algorithmic bias and expressivity for learning algorithms. This work builds on recent results in theoretical machine learning, which highlight the necessity of incorporating biases tailored to specific learning problems in order to achieve learning performance that is better than uniform random sampling of the hypothesis space~\cite{futilityOfBiasFreeLearning}. A trade-off exists between specialization and flexibility of learning algorithms. While algorithmic bias can be viewed as an algorithm's ability to `specialize', expressivity characterizes the `flexibility' of a learning algorithm. Using the algorithmic search framework for learning~\cite{montanez2017dissertation}, we define a specific form of expressivity, called \textit{entropic expressivity}, which is a function of the information-theoretic entropy of an algorithm's induced probability distribution over its search space. Under this notion of expressivity, the degree to which a search algorithm is able to spread its probability mass on many distinct target sets captures the extent to which the same algorithm is said to be capable of `expressing' a preference towards different search outcomes. No algorithm can be both highly biased and highly expressive.

\section{\uppercase{Related Work}}
\noindent
Inspired by Mitchell's work highlighting the importance of incorporating biases in classification algorithms to generalize beyond training data~\cite{needforbiases}, we propose a method to measure algorithmic expressivity in terms of the amount of bias induced by a learning algorithm. This paper delves further into the relationships between algorithmic bias and expressivity by building on the search and bias theoretical frameworks defined in \cite{futilityOfBiasFreeLearning}. Monta\~nez et al.\ proved that bias is necessary for a learning algorithm to perform better than uniform random sampling, and algorithmic bias was shown to encode trade-offs, such that no algorithm can be concurrently biased towards many distinct target sets. In this paper, we apply these properties of algorithmic bias to derive an upper bound on the level of bias encoded in a learning algorithm, in order to gain  insights on the expressivity of learning algorithms. 

Within statistical learning literature, there exists various measures characterizing algorithmic expressivity. For instance, the Vapnik-Chervonekis (VC) dimension \cite{vcDimension} provides a loose upper bound on algorithmic expressivity in general by characterizing the number of data points that can be exactly classified by the learning algorithm, for any possible labeling of the points. However, the disadvantages of the VC dimension include its inherent dependence on the dimensionality of the space on which the learning algorithm operates on~\cite{expressivityApplications}, as well as the fact that it is only restricted to classification problems. Building on the original VC dimension idea, Kearns and Schapire developed a generalization of the VC dimension with the Fat-shattering VC dimension by deriving dimension-free bounds with the assumption that the learning algorithm operates within a restricted space~\cite{fatVcDimension}. Further, Bartlett and Mendelson created Rademacher complexity as a more general measure of algorithmic expressivity by eliminating the assumption that learning algorithms are restricted within a particular distribution space~\cite{RademacherComplexityA}.

In this paper, we establish an alternative general measure of algorithmic expressivity based on the algorithmic search framework~\cite{montanez2017fof}. Because this search framework applies to clustering and optimization~\cite{montanez2017dissertation} as well as to the general machine learning problems considered in Vapnik's learning framework~\cite{vapnik1999overview}, such as classification, regression, and density estimation, theoretical derivations of the expressivity of search algorithms using this framework directly apply to the expressivity of many types of learning algorithms. 

\section{\uppercase{Search Framework}}

\subsection{The Search Problem}
We formulate machine learning problems as search problems using the algorithmic search framework \cite{montanez2017fof}. Within the framework, a search problem is represented as a 3-tuple $(\mathrm{\Omega}, T, F)$. The finite \textbf{search space} from which we can sample is $\mathrm{\Omega}$. The subset of elements in the search space that we are searching for is the \textbf{target set} $T$. A \textbf{target function} that represents $T$ is an $|\mathrm{\Omega}|$-length vector with entries having value 1 when the corresponding elements of $\mathrm{\Omega}$ are in the target set and 0 otherwise. The \textbf{external information resource} $F$ is a finite binary string that provides initialization information for the search and evaluates points in $\mathrm{\Omega}$, acting as an oracle that guides the search process. In learning scenarios this is typically a dataset with accompanying loss function.

\subsection{The Search Algorithm}
Given a search problem, a history of elements already examined, and information resource evaluations, an algorithmic search is a process that decides how to next query elements of $\mathrm{\Omega}$. As the search algorithm samples, it adds the record of points queried and information resource evaluations, indexed by time, to the search history. The algorithm uses the history to update its sampling distribution on $\mathrm{\Omega}$. An algorithm is successful if it queries an element $\omega \in T$ during the course of its search. Figure \ref{fig:jellyfish} visualizes the search process.

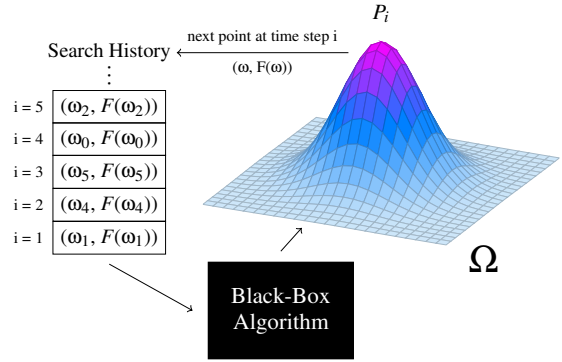
\begin{figure}
    \resizebox{7.5cm}{!}{\def\centerx{2}
\def\centery{-1}
\begin{tikzpicture}[scale=0.8]
    \begin{axis}[hide axis]
    \addplot3[surf, domain=-2:6,domain y=-5:3] 
        {exp(-( (x-\centerx)^2 + (y-\centery)^2)/3 )};
    \node[text centered] at (axis cs:\centerx, \centery, 1.20) {\large $P_i$};
    \end{axis}
    \node at (5.4,0) {\huge{$\mathrm{\Omega}$}};
    \draw[->] (\centerx + 0.8,4) -- node[above, text centered] {\scriptsize next point at time step i} node[below, text centered] {\scriptsize ($\omega$, F($\omega$))} (-0.5,4);
    \node[draw, fill=black!, text=white, text width=2cm, minimum height=1.5cm, text centered] at (1.5,-1) {Black-Box Algorithm};
    \draw[->] (1.5, 0.15) -- (1.9, 0.6  );
    \node[minimum width=1.65cm, text centered] at (-1.8,4.01) 
        {\small Search History}; 
     \foreach \y in {3.4,3.55, 3.7} 
        \node[ minimum width=2cm, minimum height=0.3cm, text centered] at (-1.8,\y) {$\cdot$};
    \node[draw, minimum width=1.65cm, text centered] at (-1.8,2.95) 
        {\footnotesize ($\omega_2$, \textit{F}($\omega_2$))}; 
    \node[text centered] at (-3.35,2.95) {\scriptsize i = 5};
    \node[draw, minimum width=1.65cm, text centered] at (-1.8,2.325) 
        {\footnotesize ($\omega_0$, \textit{F}($\omega_0$))};
    \node[text centered] at (-3.35,2.325) {\scriptsize i = 4};
    \node[draw, minimum width=1.65cm, text centered] at (-1.8,1.69) 
        {\footnotesize ($\omega_5$, \textit{F}($\omega_5$))};
    \node[text centered] at (-3.35,1.69) {\scriptsize i = 3};
    \node[draw, minimum width=1.65cm, text centered] at (-1.8,1.05) 
        {\footnotesize ($\omega_4$, \textit{F}($\omega_4$))};
    \node[text centered] at (-3.35,1.05) {\scriptsize i = 2};    
    \node[draw, minimum width=1.65cm, text centered] at (-1.8,0.41) 
        {\footnotesize ($\omega_1$, \textit{F}($\omega_1$))};
    \node[text centered] at (-3.35,0.41) {\scriptsize i = 1};
    \draw[->] (-1.8, -0.1) -- (-0.2, -1);
\end{tikzpicture}}
    \caption{As a black-box optimization algorithm samples from $\mathrm{\Omega}$, it produces an  associated probability distribution $P_i$ based on the search history. When a sample $\omega_k$ corresponding to location $k$ in $\mathrm{\Omega}$ is evaluated using the external information resource $F$, the tuple ($\omega_k$, $F(\omega_k)$) is added to the search history.}
    \label{fig:jellyfish}
\end{figure}

\subsection{Measuring Performance}
Following Monta\~nez, we measure a learning algorithm's performance using the expected per-query probability of success~\cite{montanez2017fof}. This quantity gives a normalized measure of performance compared to an algorithm's total probability of success, since the number of sampling steps may vary depending on the algorithm used and the particular run of the algorithm, which in turn effects the total probability of success. Furthermore, the per-query probability of success naturally accounts for sampling procedures that may involve repeatedly sampling the same points in the search space, as is the case with genetic algorithms \cite{goldberg1999genetic,reeves2002genetic}, allowing this measure to deftly handle search algorithms that manage trade-offs between exploration and exploitation. 

The expected per-query probability of success is defined as
\[ q(T,F) = \mathbb{E}_{\tilde{P}, H} \Bigg[ \frac{1}{|\tilde{P}|} \sum_{i=1}^{|\tilde{P}|} P_i(\omega \in T) \Bigg| F \Bigg] \]
where $\tilde{P}$ is a sequence of probability distributions over the search space (where each timestep \(i\) produces a distribution $P_i$), \(T\) is the target, \(F\) is the information resource, and \(H\) is the search history. The number of queries during a search is equal to the length of the probability distribution sequence,  $|\tilde{P}|$. The outer expectation accounts for stochastic differences in multiple runs of the algorithm, whereas the inner quantity is equivalent to the expected probability of success for a uniformly sampled time step of a given run.
    
\section{\uppercase{Bias}}
\noindent
In this section, we review the definition of bias introduced in \cite{futilityOfBiasFreeLearning} and restate some results related to that concept, showing the need for bias in learning algorithms.

\begin{definition}
    \label{def:bias_D}
    (Bias between a distribution over information resources and a fixed target) Let $\mathcal{D}$ be a distribution over a  space of information resources $\mathcal{F}$ and let $F \sim \mathcal{D}$. For a given  $\mathcal{D}$ and a fixed $k$-hot target function \(\bm{t}\) (corresponding to target set $t$),
    \begin{align*}
        \bias(\mathcal{D}, \bm{t})
        &= \E_{\mathcal{D}}[q(t,F)] - \frac{k}{|\Omega|}\\
        &= \mathbb{E}_{\mathcal{D}} \left[\bm{t}^\top \overline{P}_{F}\right] - \frac{\|\bm{t}\|^2}{|\Omega|} \\
        &= \bm{t}^\top \mathbb{E}_{\mathcal{D}}\left[\,\overline{P}_{F}\right] -    \frac{\|\bm{t}\|^2}{|\Omega|} \\
        &= \bm{t}^\top  \int_{\mathcal{F}} \overline{P}_{f} \mathcal{D}(f) \dif f - \frac{\|\bm{t}\|^2}{|\Omega|}
    \end{align*}
    where $\overline{P}_{f}$ is the vector representation of the averaged probability distribution (conditioned on $f$) induced on $\Omega$ during the course of the search, which implies $q(t,f) = \bm{t}^\top \overline{P}_{f}$.
\end{definition}

\begin{definition}
    \label{def:bias_B}
    (Bias between a finite set of information resources and a fixed target) Let $\mathcal{U}[\mathcal{B}]$ denote a uniform distribution over a finite set of information resources \(\mathcal{B}\). For a random quantity $F \sim \mathcal{U}[\mathcal{B}]$, the averaged \(|\Omega|\)-length simplex vector $\overline{P}_{F}$, and a fixed $k$-hot target function \(\bm{t}\), 
    \begin{align*}
        \bias(\mathcal{B}, \bm{t}) 
        &= \mathbb{E}_{\mathcal{U}[\mathcal{B}]}[\bm{t}^\top \overline{P}_{F}] - \frac{k}{|\Omega|} \\
        &= \bm{t}^\top \mathbb{E}_{\mathcal{U}[\mathcal{B}]}[\overline{P}_{F}] - \frac{k}{|\Omega|} \\
        &= \bm{t}^\top \left( \frac{1}{|\mathcal{B}|}\sum_{f \in \mathcal{B}} \overline{P}_{f} \right) - \frac{\|\bm{t}\|^{2}}{|\Omega|}.
    \end{align*}
\end{definition}

\begin{restatable}[Improbability of Favorable Information Resources]{theorem}{iofir}
    Let $\mathcal{D}$ be a distribution over a set of information resources $\mathcal{F}$, let $F$ be a random variable such that $F \sim \mathcal{D}$, let $t \subseteq \Omega$ be an arbitrary fixed $k$-sized target set with corresponding target function $\bm{t}$, and let $q(t,F)$ be the expected per-query probability of success for algorithm $\mathcal{A}$ on search problem $(\Omega,t,F)$. Then, for any $q_{\mathrm{min}} \in [0,1]$,
    \begin{align*}
        \Pr(q(t, F) \geq q_\mathrm{min})  &\leq \frac{p + \bias(\mathcal{D}, \bm{t})}{q_{\mathrm{min}}}
    \end{align*}
    where $p = \frac{k}{|\Omega|}$.
    \label{thm:iofir}
\end{restatable}
\noindent

\begin{restatable}[Conservation of Bias]{theorem}{conservation}
    Let $\mathcal{D}$ be a distribution over a set of information resources and let $\tau_{k} = \{ \bm{t} | \bm{t} \in \{ 0, 1 \}^{ |\Omega| }, ||\bm{t}|| = \sqrt{k} \}$ be the set of all $|\Omega|$-length $k$-hot vectors. Then for any fixed algorithm $\mathcal{A}$, 
    \begin{align*}
        \sum_{\bm{t} \in \tau_{k}} \bias(\mathcal{D},\bm{t}) = 0
    \end{align*}
    \label{thm:consbias}
\end{restatable}

\begin{restatable}[Famine of Favorable Information Resources]{theorem}{fofir}
    Let $\mathcal{B}$ be a finite set of information resources and let $t \subseteq \Omega$ be an arbitrary fixed $k$-size target set with corresponding target function $\bm{t}$. Define 
    \begin{align*}
        \mathcal{B}_{q_{\mathrm{min}}} &= \{f \mid f \in \mathcal{B}, q(t,f) \geq q_{\mathrm{min}} \},
    \end{align*}
    where $q(t,f)$ is the expected per-query probability of success for algorithm $\mathcal{A}$ on search problem $(\Omega, t,f)$ and $q_{\mathrm{min}} \in [0,1]$ represents the minimum acceptable per-query probability of success. Then,
    \begin{align*}
        \frac{|\mathcal{B}_{q_{\mathrm{min}}}|}{|\mathcal{B}|} &\leq \frac{p +  \bias(\mathcal{B}, \bm{t})}{q_{\mathrm{min}}}
    \end{align*}
    where $p = \frac{k}{|\Omega|}$.
    \label{thm:fofir}
\end{restatable}

\begin{restatable}[Futility of Bias-Free Search]{theorem}{futility}\label{thm:futility}
    For any fixed algorithm $\mathcal{A}$, fixed target $t \subseteq \Omega$ with corresponding target function $\bm{t}$, and distribution over information resources $\mathcal{D}$, if $\bias(\mathcal{D}, \bm{t}) = 0$, then
    \begin{align*}
        \Pr(\omega \in t; \mathcal{A}) &= p
    \end{align*}
    where $\Pr(\omega \in t; \mathcal{A})$ represents the single-query probability of successfully sampling an element of $t$ using $\mathcal{A}$, marginalized over information resources $F \sim \mathcal{D}$, and $p$ is the single-query probability of success under uniform random sampling.
\end{restatable}
\section{\uppercase{Main Results}}

\noindent Having reviewed the definitions of bias and prior results related to it, we now present our own results, with full proofs given in the Appendix. We proceed by presenting new results regarding bias and defining entropic expressivity. We explore expressivity in relation to bias, demonstrating a trade-off between them.

\begin{restatable}[Bias Upper Bound]{theorem}{biasUpperBound}
    \label{thm:biasUpperBound}
    Let $\tau_{k} = \{ \bm{t} | \bm{t} \in \{ 0, 1 \}^{ |\Omega| }, ||\bm{t}|| = \sqrt{k} \}$ be the set of all $|\Omega|$-length $k$-hot vectors and let $\mathcal{B}$ be a finite set of information resources. Then,
    \begin{align*}
        \lub{\mathbf{t} \in \tau_k} \bias(\mathcal{B}, \mathbf{t}) &\leq \bigg(\frac{p-1}{p}\bigg)\inf_{\mathbf{t} \in \tau_k} \bias(\mathcal{B},\mathbf{t})
    \end{align*}
    where $p = \frac{k}{|\Omega|}$.
\end{restatable}
Theorem~\ref{thm:biasUpperBound} confirms the intuition that the bounds on the maximum and minimum values the bias can take over all possible target sets are related by at most a constant factor. Note that from this theorem we can also derive a lower bound on the infimum of the bias by simply dividing by the constant factor.

We also consider the bound's behavior as p varies in Figure \ref{fig:change_p}. As $p$ increases, which can only happen as the size of the target set $k$ increases relative to the size of $\rm{\Omega}$, the upper bound on bias tightens. This is because if the target set size is a great proportion of the search space, it is more likely that the algorithm will do well on a greater number of target sets. Thus, it will be less biased towards any given one of them, by conservation of bias (Theorem~\ref{thm:consbias}).

\begin{figure}
    \centering
    \scalebox{0.8}{\definecolor{lightblue}{RGB}{205, 230, 249}
\definecolor{darkblue}{RGB}{0, 90, 180}
\definecolor{purple}{RGB}{0, 38, 77}

\begin{tikzpicture}
\begin{axis}%
    [
        xtick={-10,-9,...,10},   
        xmin=0,
        xmax=1,
        xlabel=$p$,
        ytick={-10,-9,...,10},
        tick label style={font=\tiny},
        ymin=0,
        ymax=1,
        ylabel={\scriptsize Upper Bound of Bias($D,t$)},
        samples=300,
        domain=0:1,
    ]
    \addplot[color = lightblue, thick] {((x-1)/x)*-0.3};
    \addplot[color = darkblue, thick] {((x-1)/x)*-0.1};
    \addplot[color = purple, thick] {((x-1)/x)*-0.01};
    
    \addlegendentry{$m = -0.3$}
    \addlegendentry{$m = -0.1$}
    \addlegendentry{$m = -0.01$}
\end{axis} 
\end{tikzpicture}}
    \caption{This graph shows how the upper bound of the supremum of the bias over all possible target sets of size $k$ varies with different values of $p$, for different values of $m = \frac{p-1}{p}$. }
    \label{fig:change_p}
\end{figure}
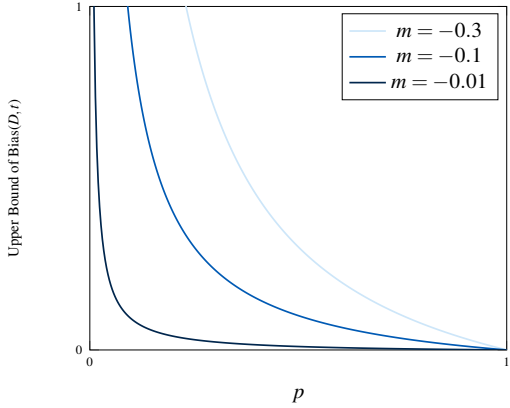

\begin{restatable}[Difference Between Estimated and Actual Bias]{theorem}{biasDifference}
    \label{def:biasDifference}
    Let $\bm{t}$ be a fixed target function, let $\mathcal{D}$ be a distribution over a set of information resources $\mathcal{B}$, and let $X = \{X_1, \dots, X_n\}$ be a finite sample independently drawn from $\mathcal{D}$. Then, 
    \begin{align*}
        \mathbbm{P}(|\bias(X, \bm{t}) - \bias(\mathcal{D}, \bm{t})| \geq \epsilon) 
            &\leq 2 e^{-2 n\epsilon^2}.
    \end{align*}
\end{restatable}
This theorem bounds the difference in the bias defined with respect to a distribution over information resources, $\bias(\mathcal{D}, \bm{t})$, and the bias defined on a finite set of information resources sampled from $\D$. In practice, we may not have access to the underlying distribution of information resources but we may be able to sample from such an unknown distribution. This theorem tells us how close empirically computed values of bias will be to the true value of bias, with high probability. 

\begin{definition}[Entropic Expressivity]
    \label{def:expressivity}
    Given a distribution over information resources $\mathcal{D}$, we define the \textit{entropic expressivity} of a search algorithm as the information-theoretic entropy of the averaged strategy distributions over $\mathcal{D}$, namely,
    \begin{align*}
        H(\overline{P}_{\D}) 
            &= H\big(\mathbbm{E}_{\D}[\overline{P}_{F}]\big) \\
            &= H(\mathcal{U}) - D_{\mathrm{KL}}(\overline{P}_{\D} \;||\; \mathcal{U})
    \end{align*}
    where $F \sim \D$ and the quantity $D_{\mathrm{KL}}(\overline{P}_{\D} \;||\; \mathcal{U})$ is the Kullback-Leibler divergence between distribution $\overline{P}_{\D}$ and the uniform distribution $\mathcal{U}$, both being distributions over search space $\Omega$.
\end{definition}
Definition~\ref{def:expressivity} uses the standard information-theoretic entropy for discrete probability mass functions, $H(\cdot)$. Our notion of expressivity characterizes the flexibility of an algorithm by measuring the entropy of its induced probability vectors (strategies) averaged over the distribution on information resources.  Algorithms that place probability mass on many different regions of the search space will tend to have a more uniform averaged probability vector. Entropic expressivity captures this key aspect of the flexibility of an algorithm.

We now present results relating this notion of expressivity to algorithmic bias.

\begin{restatable}[Expressivity Bounded by Bias]{theorem}{expressivityBiasBound}
    \label{thm:expressivityBiasBound}
    Given a fixed $k$-hot  target function $\bm{t}$ and a distribution over information resources $\D$, the entropic expressivity of a search algorithm can be bounded in terms of $\epsilon := \bias(\D, \bm{t})$, by
    \begin{align*}
        H(\overline{P}_{\D}) \in 
            &\bigg[H(p + \epsilon), \bigg((p + \epsilon) \log_2 \bigg(\frac{k}{p + \epsilon}\bigg) \\
            &+ (1 - (p + \epsilon)) \log_2 \bigg(\frac{|\Omega| - k}{1 - (p + \epsilon)}\bigg)\bigg)\bigg].
    \end{align*}
\end{restatable}
This theorem shows that entropic expressivity is bounded above and below with respect to the level of bias on a fixed target. Table~\ref{tab:expressivityRanges} demonstrates the different expressivity ranges for varying levels of bias. While these ranges may be quite large, maximizing the level of bias significantly reduces the range of possible values of entropic expressivity. 

\begin{table}
    \centering
    \caption{Varying ranges of entropic expressivity for different levels of bias on target $\bm{t}$.}
    \begin{TAB}(r,0.3cm,1cm)[2pt]{|c|c|c|}{|c|c|c|c|}
        $\bias(\D, \bm{t})$ & $\mathbbm{E}[\bm{t}^\top \overline{P}_F]$ & \textbf{Expressivity Range}\\
        \begin{tabular}{@{}c@{}} $-p$ \\ (Minimum bias) \end{tabular} & $0$ &  $[0, \log_2( |\Omega| - k)]$ \\
        \begin{tabular}{@{}c@{}} $0$ \\ (No bias) \end{tabular} & $p$ & $[H(p), \log_2 |\Omega|]$ \\
        \begin{tabular}{@{}c@{}} $1-p$ \\ (Maximum bias) \end{tabular} & $1$ & $[0, \log_2 k]$ \\
    \end{TAB}
    \label{tab:expressivityRanges}
\end{table}

\begin{restatable}[Bias-Expressivity Trade-off]{theorem}{tradeoff}
    \label{thm:tradeoff}
    Given a distribution over information resources $\D$ and a fixed target $t \subseteq \Omega$, entropic expressivity is bounded above in terms of bias,
    $$H(\overline{P}_{\D}) \leq \log_2 |\Omega| - 2 \bias(\D, \bm{t})^2$$ 
    Additionally, bias is bounded above in terms of entropic expressivity,
    \begin{align*}
        \bias(\D,\bm{t}) 
            &\leq \sqrt{\frac{1}{2}(\log_2|\Omega| - H(\overline{P}_{\D}))} \\
            &= \sqrt{\frac{1}{2} D_{\text{KL}}(\overline{P}_{\D} \;||\; \mathcal{U})}.
    \end{align*} 
\end{restatable}
Theorem~\ref{thm:tradeoff} demonstrates a trade-off between bias and entropic expressivity. We bound entropic expressivity above in terms of bias and bias above in terms of entropic expressivity such that higher values of bias decrease the range of possible values of expressivity and higher values of expressivity decrease the range of possible values of bias. Thus, a higher level of bias on a specified target restricts the expressivity of the underlying strategy distribution and a higher level of expressivity on the underlying strategy distribution restricts the amount of bias on any arbitrary target. Intuitively, this trade-off means that preferences towards specific targets reduces the potential flexibility of our algorithm over all elements and vice versa.

Lastly, we give a corollary bound allowing us to bound bias as a function of the expected entropy of induced strategy distributions, rather than the entropic expressivity.

\begin{restatable}[Bias Bound Under Expected Expressivity]{corollary}{jensen}
    \label{cor:jensen}
    \begin{align*}
        \bias(\D,\bm{t}) 
            &\leq \sqrt{\frac{1}{2}(\log_2|\Omega| - \E_{\D}[H(\overline{P}_{F})])} \\
            &= \sqrt{ \E_{\D}\left[\frac{1}{2}D_{\text{KL}}(\overline{P}_{F} \;||\; \mathcal{U})\right]}.
    \end{align*}
\end{restatable}

\section*{\uppercase{Conclusion}}
\noindent
Expanding results on the algorithmic search framework, we supplement the notion of bias and define entropic expressivity, as well as its relation to bias. We upper bound the bias on an arbitrary target set with respect to the minimum bias toward a target set over all possible target sets of a fixed size. Moreover, we upper bound the probability of the difference between the estimated bias and the true bias exceeding some threshold, showing an exponential rate of measure concentration in the number of samples. Entropic expressivity characterizes the degree of uniformity for strategy distributions in expectation for an underlying distribution of information resources. We provide upper and lower bounds of the entropic expressivity with respect to the bias on a specified target and we demonstrate a trade-off between bias and expressivity.

While bias is needed for better-than-chance performance of learning algorithms, bias also hinders the flexibility of an algorithm by reducing the different ways it can respond to varied training data. Although algorithms predisposed to certain outcomes  will not adapt as well as algorithms without strong predispositions,  maximally flexible algorithms (those without any bias) can only perform as well as uniform random sampling (Theorem~\ref{thm:futility}). This paper explores the trade-off, giving bounds for bias in terms of expressivity, and bounds for expressivity in terms of bias, demonstrating that such a trade-off exists. Although the notions of bias are different, the bias-expressivity trade-off can be viewed as a type of bias-variance trade-off~\cite{geman1992neural,kohavi1996bias}, where bias here is not an expected error but an expected deviation from uniform random sampling performance caused by an algorithm's inductive assumptions, and variance is not a fluctuation in observed error caused by changing data but is instead a ``fluctuation'' in algorithm outcome distributions caused by the same. Therefore, our results may provide new insights for that well-studied phenomenon.

\newpage

\bibliographystyle{apalike}
{\small
\bibliography{bibliography}}
\section*{\uppercase{Appendix}}

\begin{lemma}[Existence of subset with at most uniform mass]
    \label{lem:offset}
    Given an $n$-sized subset $S$ of the sample space of an arbitrary probability distribution with total probability mass $M_S$, there exists a $k$-sized proper subset $R \subset S$ with total probability mass $M_R$ such that $$M_R \leq \frac{k}{n}M_S.$$
\end{lemma}

\begin{proof} We proceed by induction on the size $k$.\\\\
    \textbf{Base Case}: When $k=1$, there exists an element with total probability mass at most $\frac{M_S}{n}$, since for any element in $S$ that has probability mass greater than the uniform mass $\frac{M_S}{n}$, there exists an element with mass strictly less than $\frac{M_S}{n}$ by the law of total probability. This establishes our base case.\\
    
    \noindent \textbf{Inductive Hypothesis}: Suppose that a $k$-sized subset $R_k \subset S$ exists with total probability mass $M_{R_k}$ such that $M_{R_k} \leq \frac{k}{n}M_S$.\\
    
    \noindent \textbf{Induction Step:} We show that there exists a subset $R_{k+1} \subset S$ of size $k+1$ with total probability mass $M_{R_{k+1}}$ such that $M_{R_{k+1}} \leq \frac{k+1}{n}M_S$.
    
    First, let $M_{R_k} = \frac{k}{n}M_S - s$, where $s \geq 0$ represents the slack between $M_{R_k}$ and $\frac{k}{n}M_S$. Then, the total probability mass on ${R_k}^\mathrm{c} := S \setminus R_k$ is $$M_{{R_k}^\mathrm{c}}= M_S - M_{R_k} = M_S - \frac{k}{n}M_S + s.$$
    Given that $M_{{R_k}^\mathrm{c}}$ is the total probability mass on set ${R_k}^\mathrm{c}$, either each of the $n-k$ elements in ${R_k}^\mathrm{c}$ has a uniform mass of $M_{{R_k}^\mathrm{c}}/(n-k)$, or they do not. If the probability mass is uniformly distributed, let $e$ be an element with mass exactly $M_{{R_k}^\mathrm{c}}/(n-k)$. Otherwise, for any element $e'$ with mass greater than $M_{{R_k}^\mathrm{c}}/(n-k)$, by the law of total probability there exists an element $e \in {R_k}^\mathrm{c}$ with mass less than $M_{{R_k}^\mathrm{c}}/(n-k)$. Thus, in either case there exists an element $e \in  {R_k}^\mathrm{c}$ with mass at most $M_{{R_k}^\mathrm{c}}/(n-k)$.
    
    Then, the set $R_{k+1} = R_k \cup \{e\}$ has total probability mass
    \begin{align*}
        M_{R_{k+1}} &\leq M_{R_k} + \frac{M_{{R_k}^\mathrm{c}}}{n-k} \\
                    &= \frac{k}{n}M_S - s + \frac{M_S - \frac{k}{n}M_S + s}{n-k} \\
                    &=  \frac{kM_S(n-k) + n(M_S - \frac{k}{n}M_S + s)}{n(n-k)} - s \\
                    &= \frac{knM_S - k^2M_S +n M_S - kM_S + ns}{n(n-k)} - s \\
                    &= \frac{(n-k)(kM_S + M_S) + ns}{n(n-k)} - s \\
                    &= \frac{k+1}{n}M_S + \frac{s}{n-k} - s \\
                    &= \frac{k+1}{n}M_S + \frac{s(1 +k - n)}{n-k} \\
                    &\leq \frac{k+1}{n}M_S 
    \end{align*}
    where the final inequality comes from the fact that $k < n$. Thus, if a $k$-sized subset $R_k \in S$ exists such that $M_{R_k} \leq \frac{k}{n} M_S$, a $k+1$-sized subset $R_{k+1} \in S$ exists such that $M_{R_{k+1}} \leq \frac{k+1}{n} M_S$. \\\\
    Since the base case holds true for $k=1$ and the inductive hypothesis implies that this rule holds for $k+1$, we can always find a $k$-sized subset $R_k \in S$ such that $$M_{R_k} \leq \frac{k}{n}M_S.$$
\end{proof}

\begin{lemma}[Maximum probability mass over a target set]
    \label{lem:max}
    Let $\tau_{k} = \{ \bm{t} | \bm{t} \in \{ 0, 1 \}^{ |\Omega| }, ||\bm{t}|| = \sqrt{k} \}$ be the set of all $|\Omega|$-length $k$-hot vectors. Given an arbitrary probability distribution $P$, 
    \begin{align*}
        \lub{\mathbf{t} \in \tau_k} \mathbf{t}^\top P \leq 1 - \bigg( \frac{1-p}{p} \bigg) \inf_{\mathbf{t} \in \tau_k} \mathbf{t}^\top P
    \end{align*}
    where $p = \frac{k}{|\Omega|}$.
\end{lemma}

\begin{proof}
    We proceed by contradiction. Suppose that 
    \begin{align*}
        \lub{\mathbf{t} \in \tau_k}\mathbf{t}^\top P > 1 - \bigg( \frac{1-p}{p} \bigg) \inf_{\mathbf{t} \in \tau_k} \mathbf{t}^\top P.
    \end{align*}
    Then, there exists some target function $\mathbf{t} \in \tau_k$ such that $$\mathbf{t}^\top P > 1 - \bigg( \frac{1-p}{p} \bigg) \inf_{\mathbf{t} \in \tau_k} \mathbf{t}^\top P.$$
    Let $\mathbf{s}$ be the complementary target function to $\mathbf{t}$ such that $\mathbf{s}$ is an $|\Omega|$-length, $(|\Omega| - k)$-hot vector that takes value $1$ where $\mathbf{t}$ takes value $0$ and takes value $0$ elsewhere. Then, by the law of total probability,
    \begin{align*}
        \mathbf{s}^\top P < \bigg( \frac{1-p}{p} \bigg) \inf_{\mathbf{t} \in \tau_k} \mathbf{t}^\top P.
    \end{align*}
    By Lemma \ref{lem:offset}, there exists a $k$-sized subset of the complementary target set with total probability mass $q$ such that
    \begin{align*}
        q   &\leq \frac{k}{|\Omega|-k} (\mathbf{s}^\top P) \\
            &< \frac{k}{|\Omega|-k}\bigg( \bigg( \frac{1-p}{p} \bigg) \inf_{\mathbf{t} \in \tau_k} \mathbf{t}^\top P\bigg) \\
            &= \frac{k}{|\Omega|-k}\bigg( \bigg( \frac{|\Omega|-k}{k} \bigg) \inf_{\mathbf{t} \in \tau_k} \mathbf{t}^\top P\bigg) \\
            &= \inf_{\mathbf{t} \in \tau_k} \mathbf{t}^\top P.
    \end{align*}
    Thus, we can always find a target set with total probability mass strictly less than $\inf_{\mathbf{t} \in \tau_k} \mathbf{t}^\top P$, which is a contradiction.
    
    Therefore, we have proven that
    \begin{align*}
        \lub{\mathbf{t} \in \tau_k}\mathbf{t}^\top P \leq 1 - \bigg( \frac{1-p}{p} \bigg) \inf_{\mathbf{t} \in \tau_k} \mathbf{t}^\top P.
    \end{align*}
\end{proof}

\biasUpperBound*
\begin{proof}
    First, define $$m := \inf_{\mathbf{t} \in \tau_k} \mathbbm{E}_{\mathcal{U}[\mathcal{B}]}[\mathbf{t}^\top \overline{P}_F] = \inf_{\mathbf{t} \in \tau_k} \bias(\mathcal{B}, \mathbf{t}) + p$$ and $$M := \lub{\mathbf{t} \in \tau_k} \mathbbm{E}_{\mathcal{U}[\mathcal{B}]}[\mathbf{t}^\top \overline{P}_F] = \lub{\mathbf{t} \in \tau_k} \bias(\mathcal{B}, \mathbf{t}) + p.$$ 
    By Lemma \ref{lem:max}, 
    \begin{align*}
        M \leq 1 - \bigg( \frac{1-p}{p} \bigg) m.
    \end{align*}
    Substituting the values of $m$ and $M$,
    \begin{align*}
        \lub{\mathbf{t} \in \tau_k} \bias(\mathcal{B}, \mathbf{t}) &\leq 1 - p - \bigg( \frac{1-p}{p} \bigg) \\
        &\bigg(\inf_{\mathbf{t} \in \tau_k} \bias(\mathcal{B}, \mathbf{t}) + p \bigg) \\
        &= \bigg(\frac{p-1}{p}\bigg)\inf_{\mathbf{t} \in \tau_k} \bias(\mathcal{B},\mathbf{t}).
    \end{align*}
\end{proof}

\begin{figure}
    \centering
    \pgfplotsset{
colormap={cool}{rgb255(0cm)=(205, 230, 249); rgb255(1cm)=(0,128,255); rgb255(2cm)=(255,0,255)}
}

\definecolor{massColor}{RGB}{205,230,249}

\def\figureOneX{5}
\def\figureOneY{5}

\begin{tikzpicture}
    \draw (\figureOneX, \figureOneY) -- node[below, text centered] {\scriptsize $k$} (\figureOneX + 3, \figureOneY);
    \draw (\figureOneX + 3, \figureOneY) -- node[below, text centered] {\scriptsize $|\Omega| - k$} (\figureOneX + 7.44, \figureOneY);
    \draw[dashed] (\figureOneX + 2.72, \figureOneY + 1.5) -- (\figureOneX + 2.72, \figureOneY - 0.3);
    
    \node[minimum height=1.22cm, text centered, label={\footnotesize $p + \epsilon$}, fill=massColor] at (\figureOneX + 2.3, \figureOneY + 0.63){};
    
    \node[minimum height=1cm, text centered, label={\footnotesize 1 - ($p + \epsilon$)}, fill=massColor] at (\figureOneX + 6.8, \figureOneY + 0.52){};
\end{tikzpicture}


\def\figureTwoX{\figureOneX}
\def\figureTwoY{\figureOneY + 10}

\begin{tikzpicture}
    \draw (\figureTwoX, \figureTwoY) -- node[below, text centered] {\scriptsize $k$} (\figureTwoX + 3, \figureTwoY);
    \draw (\figureOneX + 3, \figureTwoY) -- node[below, text centered] {\scriptsize $|\Omega| - k$} (\figureTwoX + 7.44, \figureTwoY);
    \draw[dashed] (\figureTwoX + 2.72, \figureTwoY + 1.5) -- (\figureTwoX + 2.72, \figureTwoY - 0.3); 

    \foreach \x in {0.12, 0.52, 0.92, 1.32, 1.72, 2.52} 
        \node[minimum height=0.35cm, text centered, fill=massColor] at (\figureTwoX + \x, \figureTwoY + 0.2){};
    \node[minimum height=0.35cm, text centered, label={\footnotesize $\frac{p + \epsilon}{k}$}, fill=massColor] at (\figureTwoX + 2.12, \figureTwoY + 0.2){};
    
    \foreach \x in {2.92, 3.32, 3.72, 4.12, 4.52, 4.92, 5.32, 5.72, 6.12, 6.52, 7.32} 
        \node[text centered, fill=massColor] at (\figureTwoX + \x, \figureTwoY + 0.14){};
    \node[text centered, label={\footnotesize $\frac{1 - (p + \epsilon)}{|\Omega| - k}$}, fill=massColor] at (\figureTwoX + 6.92, \figureTwoY + 0.14){};
    
\end{tikzpicture}
    \caption{Assuming positive bias, this figure shows two discrete probability distributions over $\Omega$. The top is of an algorithm with high  KL divergence while the bottom is of an algorithm with low KL divergence.}
    \label{fig:fig_3}
\end{figure}
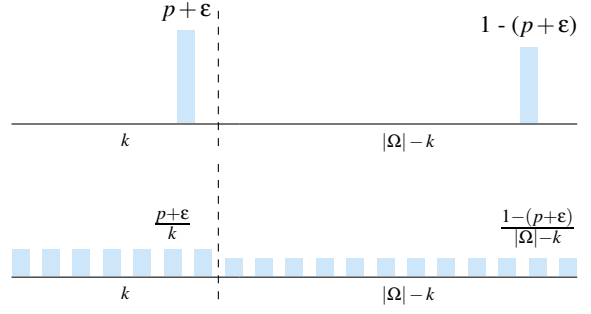

\biasDifference*
\begin{proof}
    Define 
    \begin{align*}
        \overline{B}_X &:= \frac{1}{n}\sum_{i=1}^{n} \bm{t}^{\top}\overline{P}_{X_i} \\
        &\phantom{:}= \bias(X, \bm{t}) + p.
    \end{align*}
    Given that $X$ is an iid sample from $\mathcal{D}$, we have
    \begin{align*}
        \E[\overline{B}_X] 
        &= \E\left[\frac{1}{n}\sum_{i=1}^{n} \bm{t}^{\top}\overline{P}_{X_i}\right] \\
        &= \frac{1}{n}\sum_{i=1}^{n} \E\left[\bm{t}^{\top}\overline{P}_{X_i}\right] \\
        &= \bias(\mathcal{D}, \bm{t}) + p.
    \end{align*}
    By Hoeffding's inequality and the fact that $$0 \leq \overline{B}_X \leq 1$$ we obtain
    \begin{align*}
    \mathbbm{P}(|\bias(X, \bm{t}) - \bias(\mathcal{D}, \bm{t})| \geq \epsilon) &=
        \mathbbm{P}(|\overline{B}_X - \E[\overline{B}_X]| \geq \epsilon)\\
        &\leq 2e^{-2n\epsilon^2}.
    \end{align*}
\end{proof}

\expressivityBiasBound*
\begin{proof}
    Following definition \ref{def:expressivity}, the expressivity of a search algorithm varies solely with respect to $D_{\mathrm{KL}}(\overline{P}_{\D} \;||\; \mathcal{U})$ since we always consider the same search space and thus $H(\mathcal{U})$ is a constant value. We obtain a lower bound of the expressivity by maximizing the value of $D_{\mathrm{KL}}(\overline{P}_{\D} \;||\; \mathcal{U})$ and an upper bound by minimizing this term.
    
    First, we show that $H(p + \epsilon)$ is a lower bound of expressivity by constructing a distribution that deviates the most from a uniform distribution over $\Omega$. By the definition of $\bias(\D, \bm{t})$, we place $(p + \epsilon)$ probability mass on the target set $t$ and $1 - (p + \epsilon)$ probability mass on the remaining $(n-k)$ elements of $\Omega$. We distribute the probability mass such that all of the $(p + \epsilon)$ probability mass of the target set is concentrated on a single element  and all of the $1 - (p + \epsilon)$ probability mass of the complement of the target set is concentrated on a single element. In this constructed distribution where $D_{\mathrm{KL}}(\overline{P}_{\D} \;||\; \mathcal{U})$ is maximized, the value of expressivity is 
    \begin{align*}
        H(\overline{P}_{\D}) 
            &= - \sum_{\omega \in \Omega} \overline{P}_{\D}(\omega) \log_2 \overline{P}_{\D}(\omega) \\
            &= -(p + \epsilon) \log_2(p + \epsilon) \\
            &- (1 - (p + \epsilon)) \log_2(1 - (p + \epsilon)) \\
            &= H(p + \epsilon)
    \end{align*}
    where the $H(p + \epsilon)$ is the entropy of a Bernoulli distribution with parameter $(p + \epsilon)$. The entropy of this constructed distribution gives a lower bound on expressivity, $$H(\overline{P}_{\D}) \geq H(p + \epsilon).$$
    \\
    Now, we show that $$(p + \epsilon) \log_2 \Big(\frac{k}{p + \epsilon}\Big) + (1 - (p + \epsilon)) \log_2 \Big(\frac{|\Omega| - k}{1 - (p + \epsilon)}\Big)$$ is an upper bound of expressivity by constructing a distribution that deviates the least from a uniform distribution over $\Omega$. In this case, we uniformly distribute $\frac{1}{|\Omega|}$ probability mass over the entire search space, $\Omega$. Then, to account for the $\epsilon$ level of bias, we add $\frac{\epsilon}{k}$ probability mass to elements of the target set and we remove $\frac{\epsilon}{n-k}$ probability mass to elements of the complement of the target set. In this constructed distribution where $D_{\mathrm{KL}}(\overline{P}_{\D} \;||\; \mathcal{U})$ is minimized, the value of expressivity is 
    \begin{align*}
        H(\overline{P}_{\D}) 
            &= - \sum_{\omega \in \Omega} \overline{P}_{\D}(\omega) \log_2 \overline{P}_{\D}(\omega) \\
            &= - \sum_{\omega \in t} \bigg(\frac{1}{|\Omega|} + \frac{\epsilon}{k}\bigg) \log_2 \bigg(\frac{1}{|\Omega|} + \frac{\epsilon}{k}\bigg) \\
            &- \sum_{\omega \in t^\mathrm{c}} \bigg(\frac{1}{|\Omega|} - \frac{\epsilon}{|\Omega| - k}\bigg) \log_2 \bigg(\frac{1}{|\Omega|} - \frac{\epsilon}{|\Omega| - k}\bigg) \\
            &= - \sum_{\omega \in t} \bigg( \frac{p + \epsilon}{k}\bigg) \log_2 \bigg(\frac{p + \epsilon}{k}\bigg) \\
            &- \sum_{\omega \in t^\mathrm{c}} \bigg(\frac{1 - (p + \epsilon)}{|\Omega| - k}\bigg) \log_2 \bigg(\frac{1 - (p + \epsilon)}{|\Omega| - k}\bigg) \\
            &= - k \bigg( \frac{p + \epsilon}{k}\bigg) \log_2 \bigg(\frac{p + \epsilon}{k}\bigg) \\
            &- (|\Omega| - k) \bigg(\frac{1 - (p + \epsilon)}{|\Omega| - k}\bigg) \log_2 \bigg(\frac{1 - (p + \epsilon)}{|\Omega| - k}\bigg) \\
            &= (p + \epsilon) \log_2 \bigg(\frac{k}{p + \epsilon}\bigg) \\
            &+  (1 - (p + \epsilon)) \log_2 \bigg(\frac{|\Omega| - k}{1 - (p + \epsilon)}\bigg).
    \end{align*}
    The entropy on this constructed distribution gives an upper bound on expressivity,
    \begin{align*}
        H(\overline{P}_{\D}) &\leq (p + \epsilon) \log_2 \bigg(\frac{k}{p + \epsilon}\bigg) \\
        &+  (1 - (p + \epsilon)) \log_2 \bigg(\frac{|\Omega| - k}{1 - (p + \epsilon)}\bigg).
    \end{align*}
    These two bounds give us a range of possible values of expressivity given a fixed level of bias, namely
    \begin{align*}
        H(\overline{P}_{\D}) \in &\bigg[H(p + \epsilon), \bigg((p + \epsilon) \log_2 \bigg(\frac{k}{p + \epsilon}\bigg) \\
        &+ (1 - (p + \epsilon)) \log_2 \bigg(\frac{|\Omega| - k}{1 - (p + \epsilon)}\bigg)\bigg)\bigg].
    \end{align*}
\end{proof}

\tradeoff*

\begin{proof}
    Let $\omega \in t$ denote the measurable event that $\omega$ is an element of target set $t \subseteq \Omega$, and let $\Sigma$ be the sigma algebra of measurable events. 
    First, note that 
    \begin{align*}
        \bias(\D, t)^2
            &= |\bias(\D, t)|^2 \\
            &= |\mathbf{t}^{\top}\E_{\D}[\overline{P}_{F}] - p|^2\\
            &= |\mathbf{t}^{\top}\overline{P}_{\D} - p|^2\\
            &= |\overline{P}_{\D}(\omega \in t) - p|^2 \\
            &\leq \frac{1}{2} D_{\text{KL}}(\overline{P}_{\D} \;||\; \mathcal{U})\\
            &= \frac{1}{2}( H(\mathcal{U}) - H(\overline{P}_{\D}))\\
            &= \frac{1}{2} (\log_2 |\Omega| - H(\E_{\D}[\overline{P}_{F}]))
    \end{align*}
    where the inequality is an application of Pinsker's Inequality. The quantity $D_{\text{KL}}(\overline{P}_{\D} \;||\; \mathcal{U})$ is the Kullback-Leibler divergence between distributions $\overline{P}_{\D}$ and $\mathcal{U}$, which are distributions on search space $\Omega$. \\\\
    Thus,
    \begin{align*}
        H(\E_{\D}[\overline{P}_{F}]) &\leq \log_2 |\Omega| - 2 \bias(\D, \bm{t})^2
    \end{align*}
    and
    \begin{align*}
        \bias(\D,t) 
            &\leq \sqrt{\frac{1}{2}(\log_2|\Omega| - H(\overline{P}_{\D}))} \\
            &= \sqrt{\frac{1}{2} D_{\text{KL}}(\overline{P}_{\D} \;||\; \mathcal{U})} \\
            &= \sqrt{\frac{1}{2}(\log_2|\Omega| - H(\E_{\D}[\overline{P}_{F}]))}.
    \end{align*} 
\end{proof}

\jensen*
\begin{proof}
    By the concavity of the entropy function and Jensen's Inequality, we obtain
    \begin{align*}
        \E_{\D}[H(\overline{P}_{F})] 
        &\leq H(\E_{\D}[\overline{P}_{F}])
        \leq \log_2 |\Omega| - 2 \bias(\D, \bm{t})^2.
    \end{align*}
    Thus, an upper bound of bias is
    \begin{align*}
       \bias(\D,t) &\leq  \sqrt{\frac{1}{2} D_{\text{KL}}(\overline{P}_{\D} \;||\; \mathcal{U})} \\
        &= \sqrt{\frac{1}{2}(\log_2|\Omega| - H(\E_{\D}[\overline{P}_{F}]))}\\
        &\leq \sqrt{\frac{1}{2}(\log_2|\Omega| - \E_{\D}[H([\overline{P}_{F}])])}\\
        &= \sqrt{ \E_{\D}\left[\frac{1}{2}D_{\text{KL}}(\overline{P}_{F} \;||\; \mathcal{U})\right]},
    \end{align*}
    where the final equality follows from the linearity of expectation and the definition of KL-divergence.
\end{proof}

\end{document}